\documentclass[11pt]{article}

\usepackage{fullpage}
\usepackage{amsmath}
\usepackage{amsfonts}
\usepackage{amsthm}
\usepackage{algorithm}
\usepackage{hyperref}
\usepackage[capitalize]{cleveref}
\usepackage{booktabs}

\usepackage[noend]{algpseudocode}

\newcommand{\reals}{\mathbb{R}}
\newcommand{\norm}[1]{\|#1\|}
\newcommand{\Ocal}{\mathcal{O}}
\newcommand{\bw}{\mathbf{w}}

\newcommand{\ow}{\overline \bw}

\newcommand{\ob}{\overline b}

\newcommand{\bv}{\mathbf{v}}
\newcommand{\bu}{\mathbf{u}}
\newcommand{\bx}{\mathbf{x}}

\newcommand{\ox}{\overline \bx}

\newcommand{\brac}[1]{[#1]}

\providecommand{\keywords}[1]{\textit{\textbf{Keywords\quad}#1}}

\newtheorem{lemma}{Lemma}

\title{Linear Learning with Sparse Data}
\author{Ofer Dekel\\{\tt oferd@microsoft.com}}
\begin{document}
\maketitle

\begin{abstract}
\noindent Linear predictors are especially useful when the data is
high-dimensional and sparse. One of the standard techniques used to
train a linear predictor is the \emph{Averaged Stochastic Gradient
  Descent} (ASGD) algorithm. We present an efficient implementation of
ASGD that avoids dense vector operations. We also describe a
translation invariant extension called Centered Averaged Stochastic
Gradient Descent (CASGD).
\end{abstract}

\keywords{machine learning, linear predictor, stochastic gradient descent, Polyak-Ruppert averaging, sparsity, efficient implementation}

\section{Introduction}
We are given a training set of labeled examples,
$\{(x_i,y_i)\}_{i=1}^m$, where each $x_i \in \reals^n$ is called a
\emph{feature vector} and $y_i \in \reals$ is its corresponding
label. We are also given a loss function $\ell : \reals^2 \mapsto
\reals$, defined over pairs of labels, where $\ell(p, y)$ is
understood to be the penalty associated with predicting the label $p$
when the correct label is known to be $y$. We restrict our discussion
to loss functions that are convex in their first argument. Different
choices of $\ell$ lead to different learning problems. For example,
choosing $\ell$ to be the \emph{absolute loss} or \emph{squared loss}
induces a regression problem, whereas choosing the \emph{hinge loss}
or \emph{log-loss} induces a binary classification problem (see
\cref{tab:losses} for the definitions of these loss functions).

\begin{table}[t] 
  \begin{center}
  \begin{tabular}{@{}rll@{}} \toprule
    name & definition & (sub)derivative \\ \midrule
    \text{absolute loss} & $\ell(p , y) = |p - y|$                    & $\ell'(p , y) = \begin{cases} -1 &\text{if~} p \leq y \\ 1 & \text{otherwise}\end{cases}$ \\
    \text{squared loss}  & $\ell(p , y) = \frac{1}{2}(p - y)^2$       & $\ell'(p , y) = p - y$ \\
    \text{hinge loss}    & $\ell(p , y) = \max\big\{1 - p y, 0\big\}$ & $\ell'(p , y) = \begin{cases} -y &\text{if~} p y \leq 1 \\ 0 & \text{otherwise}\end{cases}$ \\
    \text{log-loss} & $\ell(p , y) = \log\big(1 + \exp(-p y)\big)$    & $\ell'(p , y) = \frac{-y}{1 + \exp(p y)} $\\
    \bottomrule
  \end{tabular}
  \end{center}
\caption{Examples of convex loss functions and their (sub)derivatives.} \label{tab:losses}
\end{table}

A linear predictor is a pair $(\bw,b)$, where $\bw \in \reals^n$ is
called the \emph{weights} vector and $b \in \reals$ is called the
\emph{bias}. Given a feature vector $x \in \reals^n$, the linear
predictor predicts the real-valued label $\bw \cdot x + b$. Therefore,
the loss incurred by the linear predictor $(\bw,b)$ on the training
example $(x_i,y_i)$ is $\ell(\bw \cdot x_i + b\,,\, y_i)$, and the
average loss on the entire training set is
$$
\frac{1}{m}\sum_{i=1}^m \ell(\bw \cdot x_i + b\,,\, y_i)~~.
$$
To promote statistical generalization, we add a regularization term to
the average loss and arrive at the objective function
\begin{equation}\label{eqn:objective}
F(\bw,b)~=~\frac{\lambda}{2} \big( \norm{\bw}^2 + b^2 \big) ~+~ \frac{1}{m}\sum_{i=1}^m \ell(\bw \cdot x_i + b\,,\, y_i)~~,
\end{equation}
where $\lambda$ is a user-defined regularization parameter. The goal
of our algorithms is to efficiently find the linear predictor that
minimizes $F$. As mentioned above, we solve this optimization problem
using the Averaged Stochastic Gradient Descent (ASGD) algorithm.

\section{Sparse Vector Operations}
In many high dimensional machine learning problems, the feature
vectors are sparse.  Namely, only a small subset of each feature
vector's entries are non-zero.  Concretely, we assume that, on
average, there are $k$ non-zero elements in each feature vector, where
$k \ll n$.  A good example of a supervised machine learning problem
with high dimensional sparse data is text categorization using a
bag-of-words feature representation. In this setting, the dimension,
$n$, is the number of words in the dictionary, which could be in the
millions. On the other hand, the number of non-zeros in each featuer
vector, $k$, is the number of unique words in a single document, which
could be a few hundreds.

Although the feature vectors are sparse, the linear predictor that
optimizes \cref{eqn:objective} can have a dense weights vector. To
emphasize that some vectors are sparse and others are dense, we denote
dense vectors using boldface roman letters, such as $\bw$, $\bv$, and
$\bu$.

Sparse vectors can be stored using a space-efficient
representation. For example, the non-zero vector elements can be
stored as a list of index-value pairs. Moreover, many standard
operations involving sparse feature vectors can be done in $\Ocal(k)$
steps, rather than $\Ocal(n)$ steps. We call these operations
\emph{sparse vector operations} and distinguish them from the more
costly \emph{dense vector operations}. For example, if $\bv$ is a
dense vector stored in a random-access representation (such as an
array), $\alpha$ is a scalar, and $x$ is a sparse vector, then the
operation $\bv \gets \bv + \alpha x$ is a sparse vector operation:
iterate over the $k$ non-zero elements of $x$ and update the
corresponding entries in $\bv$. Similarly, calculating the dot product
$\bv \cdot x$ requires only $\Ocal(k)$ steps.

Since sparse operations are much faster than dense operations, we want
to implement ASGD using only a small constant number of dense
operations. Specifically, this implies that we can only perform sparse
vector operations inside the gradient descent loop.  More precisely,
if ASGD performs $T$ gradient descent steps, its total running time
should be $\Ocal(n + Tk)$ rather than $\Ocal(Tn)$.

\section{Stochastic Gradient Descent}
As a warm-up to ASGD, we first discuss the simpler Stochastic Gradient
Descent (SGD) algorithm \cite{robbins51,bottou91}. SGD is an
interative optimization technique that runs for $T$ steps and produces
a sequence of intermediate linear predictors $\big((\bw_t, b_t)
\big)_{t=0}^T$. The first predictor in the sequence, $(\bw_0, b_0)$,
is initialized to zero. SGD performs $T$ gradient descent steps, each
one with respect to an individual training example that is drawn
uniformly from the training set.  Formally, let $\pi_1,\ldots,\pi_T$
be a sequence of independently drawn random indices, each between $1$
and $m$; on iteration $t$ the algorithm processes training example
$\pi_t$.

To derive the SGD update, We use the square-bracket notation $[\bw,b]$
to denote the concatenation of $\bw$ and $b$. Similarly, we use
$[x,1]$ to denote the concatenation of the value $1$ to the end of the
vector $x$. The subgradient of \cref{eqn:objective} is,
\begin{equation}\label{eqn:gradient}
\nabla F(\bw, b) ~=~
\lambda \brac{\bw, b} ~+~ \frac{1}{m}\sum_{i=1}^m \ell'(\bw \cdot x_i + b\,,\, y_i) \; \brac{x_i, 1}~~.
\end{equation}
If $\pi$ is a random index, chosen uniformly between $1$ and $m$, then 
$$
\lambda \brac{\bw, b} ~+~ \ell'(\bw \cdot x_\pi + b\,,\, y_\pi) \; \brac{x_\pi, 1}
$$
is an unbiased estimator of \cref{eqn:gradient}, also called a
\emph{stochastic gradient} of the objective function in
\cref{eqn:objective}. Each SGD step subtracts a scaled stochastic
gradient from the current predictor. The algorithms allows for some
flexibility in choosing the size of each step, and we choose the size
of step $t$ to be $1/\lambda t$, where $\lambda$ is the regularization
parameter in \cref{eqn:objective}. This step size is motivated by the
theoretical convergence analysis of SGD with strongly convex objective
functions \cite{hazan07,shalev11}. Overall, the update on iteration
$t$ takes the form
\begin{align*}
  \brac{\bw_t, b_t} ~=~& \brac{\bw_{t-1}, b_{t-1}} ~-~ \frac{1}{\lambda t} \left(\lambda \, \brac{\bw_{t-1}, b_{t-1}}
  ~+~ \ell'(\bw_{t-1} \cdot x_{\pi_t} + b_{t-1} \,,\, y_{\pi_t})\; \brac{x_{\pi_t}, 1}\right)~~.
\end{align*}
Rearranging terms above gives 
\begin{align} \label{eqn:GDStep}
  \brac{\bw_t, b_t}~=~& \left(1 - \frac{1}{t}\right) \brac{\bw_{t-1}, b_{t-1}} ~-~ \frac{\ell'(p_t,\, y_{\pi_t})}{\lambda t} \; \brac{x_{\pi_t}, 1}
  \quad \text{where}\quad p_t =  \bw_{t-1} \cdot x_{\pi_t} + b_{t-1} ~~.
\end{align}
Recall that our goal is to avoid all dense vector operations when
performing each SGD step. The vector $\brac{\bw_{t-1}, b_{t-1}}$ on
the right-hand side above is likely a dense vector, and therefore a
na\"ive implementation of the scaling operation $(1 - t^{-1})
\brac{\bw_{t-1}, b_{t-1}}$ would require $\Ocal(n)$ steps. To avoid this,
we introduce the \emph{gradient sum} variable, defined for each $t$ as
\begin{equation} \label{eqn:gradSum}
\brac{\bv_t, a_t} ~=~ \sum_{j=1}^t \ell'(p_j,\, y_{\pi_j}) \; \brac{x_{\pi_j}, 1}~~.
\end{equation}
On one hand, the gradient sum can be computed using sparse vector
operations. On the other hand, we prove that the linear predictor
$\brac{\bw_t, b_t}$ can be easily recovered from $\brac{\bv_t, a_t}$.
\begin{lemma} \label{lem:v}
  Let $\brac{\bw_t, b_t}$ be as defined in \cref{eqn:GDStep} and let
  $\brac{\bv_t, a_t}$ be as defined in \cref{eqn:gradSum}. Then, it
  holds for all $t \geq 1$ that $\frac{-1}{\lambda t}\brac{\bv_t, a_t} = \brac{\bw_t, b_t}$.
\end{lemma}
\begin{proof}
It is easier to prove the equivalent opposite direction: we assume
that we defined $\brac{\bv_t, a_t} = -\lambda t \, \brac{\bw_t, b_t}$
and prove that \cref{eqn:gradSum} follows.

For $t=1$, \cref{eqn:GDStep} implies that $\brac{\bw_1, b_1} =
\frac{1}{\lambda} - \ell'(0,\, y_{\pi_t})$. Scaling both sides of this
equality by $\lambda$ and using the assumption gives $\brac{\bv_1,
  a_1} = \ell'(0,\, y_{\pi_t})$, which is consistent with
\cref{eqn:gradSum}.  For $t \geq 2$, we replace $\brac{\bw_t, b_t}$
with $\frac{-1}{\lambda t}\brac{\bv_t, a_t}$ in \cref{eqn:GDStep} to
get
\begin{align*}
  \frac{-1}{\lambda t} \brac{\bv_t, a_t} ~=~& \left(1 - \frac{1}{t} \right) \frac{-1}{\lambda (t-1)} \brac{\bv_{t-1}, a_{t-1}} ~-~ \frac{\ell'(p_t,\, y_{\pi_t})}{\lambda t} \; \brac{x_{\pi_t}, 1}~~.
\end{align*}
Using the fact that $(1-\frac{1}{t})\frac{1}{t-1} = \frac{1}{t}$, we
multiply both sizes of the equation above by $-\lambda t$ and get
$$
\brac{\bv_t, a_t} ~=~ \brac{\bv_{t-1}, a_{t-1}} ~+~ \ell'(p_t,\, y_{\pi_t}) \; \brac{x_{\pi_t}, 1}~~.
$$
This implies \cref{eqn:gradSum}, which concludes the proof.
\end{proof}
We can now rewrite the prediction $p_t$ in terms of $\bv_t$ and
$a_t$. For $t=1$, it simply holds that $p_1 = 0$. For $t \geq 2$, we
use \cref{lem:v} and get
$$
p_{t} ~=~ \frac{-1}{\lambda (t-1)} \big( \bv_{t-1} \cdot x_{\pi_t} + a_{t-1} \big)~~.
$$
We are now ready to design an efficient implementation of SGD. Our
algorithm computes the sequence of gradient sums $\big((\bv_t, a_t)
\big)_{t=0}^T$ using only sparse vector operations. Whenever needed,
the linear predictor $(\bw_t, b_t)$ can be recovered from $(\bv_t,
a_t)$ by performing a one-time dense rescaling by $\frac{-1}{\lambda
  t}$. The pseudocode for this algorithm appears in \cref{alg:sgd}.

\begin{algorithm}[t]
\begin{algorithmic}[1]
\algsetblock[Name]{rename}{}{0}{1cm}
\Function{SGD}{$T, \lambda, \{(x_t, y_t)\}_{i=1}^m$} \vbox to 12pt{\vfill} \Comment{number of steps, regularization parameter, training set}
\State draw random indices $\pi_1,\ldots,\pi_{T}$ \vbox to 12pt{\vfill}
\State $g ~\gets~ \ell'(0, y_{\pi_1 })$ \vbox to 12pt{\vfill}
\State $\bv ~\gets~ g x_{\pi_1}$ \vbox to 12pt{\vfill}
\State $a ~\gets~ g$ \vbox to 12pt{\vfill}
\For {$t = 2,\ldots,T$ \vbox to 12pt{\vfill}} 
   \State $d ~\gets~ \bv \cdot x_{\pi_t}$ \vbox to 16pt{\vfill} \Comment{$\Ocal(k)$ operation}
   \State $p ~\gets~ \frac{-(d+a)}{\lambda (t-1)}$ \vbox to 11pt{\vfill} \Comment{note that $p = \bw_{t-1} \cdot x_{\pi_{t}} + b_{t-1}$}
   \State $g ~\gets~ \ell'(p, y_{\pi_t})$ \vbox to 9pt{\vfill} 
   \State $\bv ~\gets~ \bv + g x_{\pi_t}$ \vbox to 11pt{\vfill} \Comment{$\Ocal(k)$ operation}
   \State $a ~\gets~ a + g $ \vbox to 11pt{\vfill} 
\EndFor
\State $\bw ~\gets~ \frac{-1}{\lambda T}\,\bv$ \vbox to 16pt{\vfill} \Comment{$\Ocal(n)$ operation outside the loop}
\State $b ~\gets~ \frac{-a}{\lambda T}$ \vbox to 11pt{\vfill} 
\State \Return $\brac{\bw, b}$ \vbox to 11pt{\vfill}
\EndFunction
\end{algorithmic}
\caption{SGD for regularized linear learning with sparse data}
\label{alg:sgd}
\end{algorithm}

\section{Averaged Stochastic Gradient Descent}
The SGD algorithm in \cref{alg:sgd} implicitly constructs a sequence
of intermediate linear predictors and returns the last predictor in
the sequence. Ruppert \cite{ruppert88} and Polyak
\cite{polyak90,polyak92} independently argued that the last predictor
may be suboptimal, and that the \emph{average} of the intermediate
predictors is a better choice.  Intuitively, the average predictor is
more stable than the last predictor, and this stability allows us to
prove strong convergence results.

Specifically, we define
\begin{equation}\label{eqn:naiveAverage}
  \ow_t ~=~ \frac{1}{t} \, \sum_{j=1}^t \bw_j \quad\text{and}\quad
  \ob_t ~=~ \frac{1}{t} \, \sum_{j=1}^t b_j~~,
\end{equation}
and we wish to modify \cref{alg:sgd} to return
$(\ow_T,\ob_T)$. This technique is called Averaged SGD, or ASGD.

To compute $\ob_t$, we use \cref{lem:v} and write
$$
\ob_t ~=~ \frac{1}{t} \, \sum_{j=1}^t b_j ~=~ \frac{-1}{\lambda t} \, \sum_{j=1}^t \frac{a_j}{j} ~~. 
$$
Using the above, we can modify \cref{alg:sgd} to incrementally compute the term 
\begin{equation} \label{eqn:avgBias}
  c_t ~=~ \sum_{j=1}^t \frac{a_j}{j}~~,
\end{equation}
and when needed, to recover
\begin{equation}\label{eqn:obDef}
\ob_t ~=~ \frac{-c_t}{\lambda t}~~.
\end{equation}

Computing $\ow_t$ requires more care, because the vector addition in
\cref{eqn:naiveAverage} involves dense vectors, and a straightforward
computation of $\ow_t$ would require $\Ocal(tn)$ operations. To avoid
these dense vector operations, we apply \cref{lem:v}, and get
$$
\bw_j ~=~ \frac{-1}{\lambda j} \sum_{i=1}^j \ell'(p_i,\, y_{\pi_i}) \; x_{\pi_i}~~.
$$
Plugging the above into \cref{eqn:naiveAverage} gives
\begin{align*} 
\ow_t ~&=~ \frac{1}{t} \sum_{j=1}^t \left( \frac{-1}{\lambda j} \sum_{i=1}^j \ell'(p_i,\, y_{\pi_i}) \; x_{\pi_i} \right)~~.
\end{align*}
Rearranging the order of the two sums and using $h_i = \sum_{j=1}^i
\frac{1}{j}$ to denote the $i$'th harmonic number, we get
\begin{align}   
  \ow_t ~&=~ \frac{-1}{\lambda t} \sum_{i=1}^t \left(\sum_{j=i}^t \frac{1}{j}\right) \ell'(p_i,\, y_{\pi_i}) \; x_{\pi_i} \nonumber \\
  &=~ \frac{-1}{\lambda t} \sum_{i=1}^t \left(h_t - h_{i-1}\right) \ell'(p_i,\, y_{\pi_i}) \; x_{\pi_i} \nonumber \\
  &=~ \frac{-1}{\lambda t} \left( h_t\,\bv_t - \sum_{i=1}^t h_{i-1}\, \ell'(p_i,\, y_{\pi_i}) \; x_{\pi_i} \right)~~. \label{eqn:sparseAverage}
\end{align}
We modify \cref{alg:sgd} to also incrementally compute the
\emph{harmonic gradient sum},
\begin{equation}\label{eqn:buDef}
\bu_t ~=~ \sum_{i=1}^t h_{i-1}\, \ell'(p_j,\, y_{\pi_j}) \; x_{\pi_j}~~.
\end{equation}
This definition allows us to write \cref{eqn:sparseAverage} as
\begin{equation}\label{eqn:owDef}
\ow_t ~=~ \frac{-1}{\lambda t} \big( h_t\,\bv_t - \bu_t \big)~~.
\end{equation}
With the formula above, $\ow_t$ can be recovered from $h_t$, $\bv_t$,
and $\bu_t$ when needed, via a dense vector operation. The pseudo-code
of the resulting ASGD implementation is presented in \cref{alg:asgd}.
\begin{algorithm}[t]
\begin{algorithmic}[1]
\algsetblock[Name]{rename}{}{0}{1cm}
\Function{ASGD}{$T, \lambda, \{(x_t, y_t)\}_{i=1}^m$} \vbox to 12pt{\vfill} \Comment{num of steps, regularization param, training set}
\State draw random indices $\pi_1,\ldots,\pi_{T}$ \vbox to 12pt{\vfill}
\State $g \gets \ell'(0, y_{\pi_1 })\;;\;\bv \gets g x_{\pi_1}\;;\;a \gets g$ \vbox to 11pt{\vfill} \Comment{same as SGD}
\State $\bu ~\gets~ 0^n$ \vbox to 11pt{\vfill} \Comment{see \cref{eqn:buDef}}
\State $c ~\gets~ a$ \vbox to 11pt{\vfill} \Comment{see \cref{eqn:avgBias}} 
\State $h \gets 1$ \vbox to 11pt{\vfill} \Comment{first harmonic number}
\For {$t = 2,\ldots,T$ \vbox to 12pt{\vfill}}
   \State $d \gets \bv \cdot x_{\pi_t}\;;\;p \gets \frac{-(d+a)}{\lambda (t-1)}\;;\;g \gets \ell'(p, y_{\pi_t})\;;\;\bv \gets \bv + g x_{\pi_t}\;;\;a \gets a + g$ \vbox to 16pt{\vfill} \Comment{same as SGD}
   \State $\bu ~\gets~ \bu + h g x_{\pi_t}$ \vbox to 9pt{\vfill} \Comment{$\Ocal(k)$ operation, see \cref{eqn:buDef}}
   \State $c ~\gets~ c + \frac{a}{t}$ \vbox to 11pt{\vfill} \Comment{see \cref{eqn:avgBias}} 
   \State $h ~\gets~ h + \frac{1}{t}$ \vbox to 11pt{\vfill} \Comment{$t$'th harmonic number}
\EndFor
\State $\ow ~\gets~ \frac{-1}{\lambda T}\,(h \bv - \bu)$ \vbox to 16pt{\vfill} \Comment{$\Ocal(n)$ operation outside the loop, see \cref{eqn:owDef}}
\State $\ob ~\gets~ \frac{-c}{\lambda T}$ \vbox to 11pt{\vfill} \Comment{see \cref{eqn:obDef}}
\State \Return $\brac{\ow, \ob}$ \vbox to 10pt{\vfill}
\EndFunction
\end{algorithmic}
\caption{ASGD for regularized linear learning with sparse data}
\label{alg:asgd}
\end{algorithm}

\section{Centering and Translation Invariance}
A disadvantage of the problem formulation in \cref{eqn:objective} is
that it is sensitive to translation (a.k.a. offset) of the training
data (namely, adding a constant vector to each feature vector in the
training set). The root of the problem is the term $b^2$ in
\cref{eqn:objective}, which discourages large values of $b$. There are
several different ways to make our algorithms translation invariant. A
simple but effective technique is to center the training
data. Centering is the process of computing the mean feature vector,
$\ox = \frac{1}{m} \sum_{j=1}^m x_j$, and subtracting it from each
$x_t$. After we center the data, a predictor with a bias of $b=0$ is
one that passes through the training data's center-of-mass.

If we apply the transformation $x \mapsto x - \ox$ to the training set
and train a predictor $(\bw,b)$, we must apply the same centering
transformation to new feature vectors before using $(\bw,b)$ to
predict their labels. There are two equivalent ways of doing this,
\emph{explicit centering} and \emph{implicit centering}. Explicit
centering involves two consecutive steps: first, create a centered
version of the feature vector $x' = x - \ox$; then, apply the
predictor to $x'$ and predict the value $\bw \cdot x' + b$. A
disadvantage of explicit centering is that it requires us to store
$\ox$ alongside $\bw$ and $b$, as part of the predictor definition. On
the other hand, implicit centering hides the centering transformation
in the bias term. Specifically, define a new bias term
\begin{equation}\label{eqn:explicitToImplicit}
  b' ~=~ b - \bw \cdot \ox~~,
\end{equation}
and apply the predictor directly to the original (uncentered) feature
vector $x$. In other words, the prediction is computed as $\bw \cdot x
+ b'$.  The two centering techniques are equivalent because
$$
\bw \cdot x' + b ~=~ \bw \cdot (x - \ox) + b ~=~
\bw \cdot x + (b - \bw \cdot \ox) ~=~
\bw \cdot x + b'~~.
$$
The advantage of implicit centering is that it allows us to forget
$\ox$ and to make predictions as if we had not used the centering
technique at all. However, note that $b'$ cannot be computed until
training concludes and $\bw$ is available.

The technical difficulty of centering sparse feature vectors is that
$\ox$ is likely a dense vector, and therefore each centered feature
vector, $(x_t - \ox)$, is dense as well. Therefore, explicitly
centering the entire training set would require $m$ dense operations
and would prevent us from using sparse vector operations during
training. In this section, we describe how to apply the centering
technique implicitly, without using dense operations.

Imagine repeating the entire derivation from the previous sections,
replacing every appearance of $x_{\pi_t}$ with $(x_{\pi_t} - \ox)$. In
particular, we would get
$$
  \brac{\bw_t, b_t} ~=~ \frac{-1}{\lambda t} \sum_{j=1}^t \ell'(p_j,\, y_{\pi_j}) \; \brac{(x_{\pi_j} - \ox), 1}~~.
$$
Reusing the definition of $\brac{\bv_t, a_t}$ from \cref{eqn:gradSum}, we rewrite the above as
\begin{equation}\label{eqn:wCenter}
\bw_t ~=~ \frac{-1}{\lambda t} ( \bv_t - a_t \ox )
\quad \text{and} \quad
b_t ~=~ \frac{-a_t}{\lambda t} ~~.
\end{equation}
We can now use \cref{eqn:explicitToImplicit} to calculate the bias
term of the implicit representation,
$$
b'_t ~=~ b_t - \bw_t \cdot \ox ~=~  
\frac{-a_t}{\lambda t} ~-~ \frac{-1}{\lambda t} ( \bv_t - a_t \ox ) \cdot \ox 
~=~ \frac{-1}{\lambda t} \big(a_t (1+\norm{\ox}^2) - \bv_t \cdot \ox \big)~~.
$$
Consider the amount of work it would take us to obtain each of the
terms above.  The term $a_t$ can be computed as in \cref{alg:asgd}.
The mean feature vector $\ox$ can be precomputed with $\Ocal(mk)$
operations, and $(1 + \norm{\ox}^2)$ can be precomputed using a single
dense operation. The only term that poses a potential problem is
$\bv_t \cdot \ox$, which is a dot product of two dense vectors.  To
overcome this problem, we introduce the \emph{projection sum}
variable, defined as
$$
z_t ~=~ \bv_t \cdot \ox ~=~ \sum_{j=1}^t \ell'(p_j,\, y_{\pi_j})\, \ox \cdot x_{\pi_j} ~~.
$$
On one hand, we can modify \cref{alg:asgd} to incrementally compute $z_t$, as
\begin{equation}\label{eqn:zDef}
  z_t ~=~ z_{t-1} ~+~ \ell'(p_t,\, y_{\pi_t})\, \ox \cdot x_{\pi_t} ~~.
\end{equation}
On the other hand, with $z_t$ handy, we can easily compute 
\begin{equation}\label{eqn:bTagCenter}
b'_t ~=~ \frac{-r_t}{\lambda t} 
\quad \text{where} \quad r_t ~=~ a_t (1+\norm{\ox}^2) - z_t
~~.
\end{equation}
Plugging \cref{eqn:wCenter} and \cref{eqn:bTagCenter} into $p_{t} = \bw_{t-1} \cdot x_{\pi_t} + b'_{t-1}$ gives
\begin{equation}\label{eqn:pCenterDef}
p_t ~=~ \frac{-1}{\lambda (t-1)} \big(\bv_{t-1} \cdot x_{\pi_t} + r_{t-1} - a_{t-1}\, \ox \cdot x_{\pi_t} \big)~~.
\end{equation}
We have everything we need to implement a centered version of SGD
using sparse operations, but we are really after a centered version of
ASGD, a.k.a. CASGD. Namely, we need to modify \cref{alg:asgd} to
return $(\ow_T,\ob'_T)$, where
\begin{equation}\label{eqn:centerAvg}
  \ow_t ~=~ \frac{1}{t} \, \sum_{j=1}^t \bw_j \quad\text{and}\quad
  \ob'_t ~=~ \frac{1}{t} \, \sum_{j=1}^t b'_j~~.
\end{equation}
To help us compute $\ob'_t$, we further modify \cref{alg:asgd} to incrementally compute
\begin{equation}\label{eqn:sDef}
  s_t ~=~ \sum_{j=1}^t \frac{r_j}{j}~~.
\end{equation}
The value of $\ob'_t$ can be recovered as
\begin{equation}\label{obCenterDef}
\ob'_t ~=~ \frac{1}{t} \sum_{j=1}^t b'_j ~=~ \frac{-1}{\lambda t} \sum_{j=1}^t \frac{r_j}{j} ~=~
\frac{-s_t}{\lambda t}~~.
\end{equation}
To compute $\ow_t$, we plug the definition of $\bw_t$ from
\cref{eqn:wCenter} into \cref{eqn:centerAvg} to get
\begin{align*}
  \ow_t ~&=~ \frac{1}{t} \, \sum_{j=1}^t \frac{-1}{\lambda j} ( \bv_j - a_j \ox ) \\
  &=~ \frac{-1}{\lambda t} \, \Bigg( \underbrace{\sum_{j=1}^t \frac{1}{j} \sum_{i=1}^j \ell'(p_i,\, y_{\pi_i}) \; x_{\pi_i}}_{(i)}
  - \underbrace{\sum_{j=1}^t \frac{a_j}{j} \ox}_{(ii)} \Bigg)~~.
\end{align*}
Both $(i)$ and $(ii)$ above should look familiar, as we encountered
them in the previous section. We rewrite the double sum in $(i)$ as we did in \cref{eqn:owDef},
and we rewrite the sum in $(ii)$ using \cref{eqn:avgBias}, to get
\begin{equation}\label{eqn:owCenterDef}
\ow_t ~=~ \frac{-1}{\lambda t} \big( h_t \bv_t - \bu_t - c_t \ox \big)~~.
\end{equation}
The pseudo-code the resulting CASGD implementation appears in \cref{alg:casgd}

\begin{algorithm}[t]
\begin{algorithmic}[1]
\algsetblock[Name]{rename}{}{0}{1cm}
\Function{CASGD}{$T, \lambda, \{(x_t, y_t)\}_{i=1}^m$} \vbox to 12pt{\vfill} \Comment{num of steps, regularization param, training set}
\State $\ox ~\gets~ \frac{1}{m} \sum_{i=1}^m x_i$ \vbox to 12pt{\vfill} \Comment{$\Ocal(mk)$ operation outside the loop}
\State $\theta ~\gets~ 1 + \norm{\ox}^2$ \vbox to 12pt{\vfill} \Comment{$\Ocal(n)$ operation outside the loop}
\State draw random indices $\pi_1,\ldots,\pi_{T}$ \vbox to 12pt{\vfill}
\State $g \gets \ell'(0, y_{\pi_1 })\;;\;\bv \gets g x_{\pi_1}\;;\;a \gets g\;;\;\bu \gets 0^n\;;\;c \gets a\;;\;h \gets 1$ \vbox to 12pt{\vfill} \Comment{same as ASGD}
\State $q \gets \ox \cdot x_{\pi_1}$ \vbox to 12pt{\vfill} \Comment{$\Ocal(k)$ operation}
\State $z ~\gets~ gq$ \vbox to 10pt{\vfill} \Comment{see \cref{eqn:zDef}}
\State $r ~\gets~ a\theta - z$ \vbox to 11pt{\vfill} \Comment{see \cref{eqn:bTagCenter}}
\State $s ~\gets~ r$ \vbox to 11pt{\vfill} \Comment{see \cref{eqn:sDef}}

\For {$t = 1,\ldots,T$ \vbox to 12pt{\vfill}}
   \State $d ~\gets~ \bv \cdot x_{\pi_t}$ \vbox to 16pt{\vfill} \Comment{same as ASGD}
   \State $q ~\gets~ \ox \cdot x_{\pi_t}$ \vbox to 11pt{\vfill} \Comment{$\Ocal(k)$ operation}
   \State $p ~\gets~ \frac{-(d+r-aq)}{\lambda (t-1)}$ \vbox to 11pt{\vfill} \Comment{see \cref{eqn:pCenterDef}}
   \State $g \gets \ell'(p, y_{\pi_t})\;;\;\bv \gets \bv + g x_{\pi_t}\;;\;a \gets a + g$ \vbox to 11pt{\vfill} \Comment{same as ASGD}
   \State $\bu \gets \bu + h g x_{\pi_t}\;;\;c \gets c + \frac{a}{t}\;;\;h \gets h + \frac{1}{t}$ \vbox to 11pt{\vfill} \Comment{same as ASGD}
   \State $z ~\gets~ z + gq$ \vbox to 10pt{\vfill} \Comment{see \cref{eqn:zDef}}
   \State $r ~\gets~ a\theta - z$ \vbox to 11pt{\vfill} \Comment{see \cref{eqn:bTagCenter}}
   \State $s ~\gets~ s + \frac{r}{t}$ \vbox to 11pt{\vfill} \Comment{see \cref{eqn:sDef}}
\EndFor
\State $\ow ~\gets~ \frac{-1}{\lambda T}\,\big( h \bv - \bu - c \ox \big)$ \vbox to 16pt{\vfill} \Comment{$\Ocal(n)$ operation outside the loop, see \cref{eqn:owCenterDef}}
\State $\ob' ~\gets~ \frac{-s}{\lambda T}$ \vbox to 11pt{\vfill} \Comment{see \cref{obCenterDef}}
\State \Return $\brac{\ow, \ob'}$ \vbox to 11pt{\vfill}
\EndFunction
\end{algorithmic}
\caption{CASGD for regularized linear learning with sparse data}
\label{alg:casgd}
\end{algorithm}

\pagebreak
\bibliographystyle{plain}
\bibliography{bib}

\end{document}